\documentclass[a4paper,runningheads]{llncs}

\RequirePackage[latin1]{inputenc}
\usepackage[english]{babel}

\usepackage{times}
\usepackage{helvet}
\usepackage{courier}
\usepackage{graphicx}
\usepackage{latexsym}
\usepackage{amssymb}
\usepackage{amsmath}
\usepackage{enumitem}
\usepackage{framed}
\usepackage[colorinlistoftodos, textwidth=3cm]{todonotes} 
\usepackage[spaces,hyphens]{url}
\usepackage{xspace}
\usepackage{array}
\usepackage{listings}
\usepackage{hyperref}
\usepackage{txfonts}

\title{Reasoning with Justifiable Exceptions in Contextual Hierarchies (Appendix)}

\author{
	Loris Bozzato\inst{1} \and Luciano Serafini\inst{1} \and Thomas Eiter\inst{2}
}

\institute{
  Fondazione Bruno Kessler, 
  Via Sommarive 18, 38123 Trento, Italy \\
  \and
  Institut f\"{u}r Informationssysteme, Technische Universit\"{a}t Wien,\\
  Favoritenstra\ss e 9-11, A-1040 Vienna, Austria\\
  \medskip
 \email{\texttt{\{bozzato,serafini\}@fbk.eu}, eiter@kr.tuwien.ac.at}
}

\def\isa{\sqsubseteq}
\def\CKB{\mathfrak{K}}

\def\I{\mathcal{I}}

\newcommand{\Pair}[2]{\left\langle#1,#2\right\rangle}

\renewcommand{\vec}[1]{\mathbf{#1}}

\newcommand{\ee}{{\vec{e}}}

\newcommand{\SROIQ}{\mathcal{SROIQ}}

\newcommand{\NI}{\mathrm{NI}}
\newcommand{\NR}{\mathrm{NR}}
\newcommand{\NC}{\mathrm{NC}}

\newcommand{\IC}{\mathfrak{I}}

\newcommand{\K}{\mathcal{K}}
\def\self{\mathrm{Self}}

\newcommand{\stru}[1]{\langle #1 \rangle}

\renewcommand{\o}{\sqcup}

\newcommand{\non}{\neg}

\newcommand{\subs}{\sqsubseteq}

\newcommand{\Lcal}{{\cal L}}

\newcommand{\cov}[1]{\preceq}

\renewcommand{\vec}[1]{\mathbf{#1}}

\newcommand{\CAS}{\mathit{CAS}}

\newcommand{\OVR}{\mathit{OVR}}
\newcommand{\casmap}{\chi}

\newcommand{\nop}[1]{}

\newcommand{\SROIQrl}{\mathcal{SROIQ}\text{-RL}}

\newcommand{\KB}{\mathrm{K}} 

\newcommand{\mlc}{\ml{c}}
\newcommand{\mlm}{\sml{m}}

\newcommand{\smlc}{\sml{c}}

\newcommand{\scriptmlm}{\scriptml{m}}

\newcommand{\subClass}{{\tt subClass}}
\newcommand{\ptop}{{\tt top}}
\newcommand{\subEx}{{\tt subEx}}
\newcommand{\subRole}{{\tt subRole}}

\newcommand{\supEx}{{\tt supEx}}
\newcommand{\pbot}{{\tt bot}}
\newcommand{\subConj}{{\tt subConj}}
\newcommand{\subRChain}{{\tt subRChain}}
\newcommand{\pDis}{{\tt dis}}
\newcommand{\pInv}{{\tt inv}}
\newcommand{\pIrr}{{\tt irr}}

\newcommand{\nom}{{\tt nom}}
\newcommand{\cls}{{\tt cls}}
\newcommand{\rol}{{\tt rol}}

\newcommand{\peq}{{\tt eq}}

\newcommand{\subEval}{{\tt subEval}}
\newcommand{\subEvalR}{{\tt subEvalR}}

\newcommand{\supForall}{{\tt supForall}}

\newcommand{\supLeqOne}{{\tt supLeqOne}}

\newcommand{\eval}{\textsl{eval}}

\newcommand{\N}{\boldsymbol{\mathsf{N}}}

\newcommand{\ml}[1]{\mathsf{#1}} 
\newcommand{\sml}[1]{{\mbox{\small\ensuremath{\ml{#1}}}}} 
\newcommand{\scriptml}[1]{{\mbox{\scriptsize\ensuremath{\ml{#1}}}}} 

\newcommand{\default}{{\mathrm D}}

\newcommand{\pprec}{{\tt prec}}

\newcommand{\ovr}{{\tt ovr}}
\newcommand{\naf}{\mathop{\tt not}}
\newcommand{\grd}{\mathit{grnd}}

\newcommand{\Body}{\mathit{Body}}

\newcommand{\rif}{\leftarrow}

\renewcommand{\ee}{\mathbf{e}}

\newcommand{\insta}{{\tt insta}}
\newcommand{\instd}{{\tt instd}}
\newcommand{\triplea}{{\tt triplea}}
\newcommand{\tripled}{{\tt tripled}}

\newcommand{\ninsta}{{\tt ninsta}}
\newcommand{\ntriplea}{{\tt ntriplea}}

\newcommand{\unsat}{{\tt unsat}}
\newcommand{\test}{{\tt test}}
\newcommand{\testf}{{\tt test\_fails}}
\newcommand{\nlit}{{\tt nlit}}
\newcommand{\nrel}{{\tt nrel}}
\newcommand{\main}{\ml{main}}
\newcommand{\smlmain}{\sml{main}}

\newcommand{\definst}{{\tt def\_insta}}
\newcommand{\deftriple}{{\tt def\_triplea}}
\newcommand{\defninst}{{\tt def\_ninsta}}
\newcommand{\defntriple}{{\tt def\_ntriplea}}
\newcommand{\defsubs}{{\tt def\_subclass}}
\newcommand{\defsubcnj}{{\tt def\_subcnj}}
\newcommand{\defsubex}{{\tt def\_subex}}
\newcommand{\defsupex}{{\tt def\_supex}}
\newcommand{\defsupforall}{{\tt def\_supforall}}
\newcommand{\defsupleqone}{{\tt def\_supleqone}}
\newcommand{\defsubr}{{\tt def\_subr}}
\newcommand{\defsubrc}{{\tt def\_subrc}}
\newcommand{\defdis}{{\tt def\_dis}}
\newcommand{\definv}{{\tt def\_inv}}
\newcommand{\defirr}{{\tt def\_irr}}

\newcommand{\SROIQrld}{\mathcal{SROIQ}\text{-RLD}}

\newcounter{myenumctr}

\setitemize{label=--,leftmargin=*}
\setenumerate{leftmargin=*}

\newcommand{\Cl}{\mathfrak{C}}
\newcommand{\level}{{\tt level}}

\begin{document}

\maketitle

\begin{abstract}
   This paper is an appendix to the paper ``Reasoning with Justifiable Exceptions in Contextual Hierarchies'' by 
   Bozzato, Serafini and Eiter, 2018~\cite{BozzatoSE:18a}.
   It provides further details on the language, the complexity results and 
   the datalog translation introduced in the main paper.
\end{abstract}

\section{$\SROIQ$ syntax and semantics}

Table~\ref{tab:sroiq} presents the syntax and semantics of $\SROIQ$
operators and axioms. In the table,
$A$ is any atomic concept, $C$ and $D$ are any concepts, $P$ and $R$ are any
atomic roles (and for $^*$ {\em simple} in the context of a knowledge base $\K$),
$S$ and $Q$ are any (possibly complex) roles, $a$ and $b$ are any individual constants,
and $n$ stands for any positive integer.

\begin{table}[ht] \scriptsize
\caption{Syntax and Semantics of $\SROIQ$}
\label{tab:sroiq}
\begin{center}
\begin{tabular}{l@{\quad}l}
\begin{tabular}[t]{@{\,}l@{\;}|@{\;}l@{\;}|@{\;}l@{\,}}
\hline
\hline
\textbf{Concept constructors} & Syntax	& Semantics \\
\hline&\mbox{}\\[-2ex]
atomic concept		& $A$		& $A^\I$ \\
top concept		&  $\top$  &  $\Delta^\I$  \\
bottom concept  & $\bot$ & $\emptyset$ \\
complement		& $\neg C$	& $\Delta^\I \setminus C^\I$ \\
intersection		& $C\sqcap D$	& $C^\I \cap D^\I$ \\
union     		& $C\sqcup D$	& $C^\I \cup D^\I$ \\ 
existential restriction	& $\exists R.C$	& $\left\{x\in\Delta^\I \,\left|\,
                                             \begin{array}{@{}l@{}}
					       \exists y.\Pair{x}{y}\in R^\I\\
					      \phantom{ \exists y.} \land\, y\in C^\I
					     \end{array}\right.\mkern-3mu
					   \right\}$ \\[2ex]
self restriction$^*$	& $\exists R.\self$
					  & $\left\{x\in\Delta^\I\,\left|\, \Pair{x}{x}\in
					    R^\I\right.\mkern-2mu\right\}$ \\[2ex]
universal restriction	& $\forall R.C$	& 
               $\left\{x\in\Delta^\I \,\left|\,
                 \begin{array}{@{}l@{}}
					       \forall y.\Pair{x}{y}\in R^\I\\
					       \phantom{\forall y.}  \rightarrow y\in C^\I
					     \end{array}\right.\mkern-3mu
					   \right\}$\\[3ex]
min.\,card.\,restriction$^*$ &
			  ${\geqslant}n R.C$
			  		  & $\left\{x\in\Delta^\I \,\left|\,
					     \begin{array}{@{}l@{}}
					      \sharp\{y\mid\Pair{x}{y} \in R^\I\\
					      \phantom{\sharp\{}\land\,y\in C^\I\} \geq n
					     \end{array}\right.\mkern-3mu
					    \right\}$\\[2ex]
max.\,card.\,restriction &
			  ${\leqslant}n R.C$
			  		  & $\left\{x\in\Delta^\I \,\left|\,
					     \begin{array}{@{}l@{}}
					      \sharp\{y\mid\Pair{x}{y} \in R^\I\\
					      \phantom{\sharp\{}\land\,y\in C^\I\} \leq n
					     \end{array}\right.\mkern-3mu
					    \right\}$\\[2ex]
 cardinality\,restriction$^*$ &
			  $=n R.C$
			  		  & $\left\{x\in\Delta^\I \,\left|\,
					     \begin{array}{@{}l@{}}
					      \sharp\{y\mid\Pair{x}{y} \in R^\I\\
					      \phantom{\sharp\{}\land\,y\in C^\I\} = n
					     \end{array}\right.\mkern-3mu
					    \right\}$\\[2ex]
nominal			& $\{a\}$ 	& $\big\{a^\I\big\}$\\[.5ex]
\hline
\multicolumn{3}{l}{}\\[1.5ex]
\end{tabular}&\begin{tabular}[t]{@{\,}l@{\;}|@{\;}l@{\;}|@{\;}l@{\,}}
\hline
\hline
\textbf{Role constructors}	& Syntax	& Semantics \\
\hline&\mbox{}\\[-2ex]
atomic role		& $R$		& $R^\I$\\
inverse role		& $R^-$		& $\left\{\Pair{y}{x} \,\left|\,
                                           \Pair{x}{y}\in R^\I\right.\mkern-2mu\right\}$\\[.2ex]
role composition	& $S{\circ}Q$	& $\left\{\Pair{x}{z} \,\left|\,
                                           \Pair{x}{y}\in S^\I, \Pair{y}{z}\in Q^\I\right .\mkern-2mu\right\}$\\[.5ex]
\hline
\hline
\textbf{Axioms}	& Syntax	& Semantics \\
\hline&\mbox{}\\[-2ex]
concept inclusion 
	& $C\isa D$	& $C^\I \subseteq D^\I$\\
concept definition	& $C\equiv D$	& $C^\I = D^\I$\\
role inclusion
	& $S \isa R$	& $S^\I \subseteq R^\I$\\
role disjointness$^*$	& $\mathrm{Dis}(P,R)$ & $P^\I\cap R^\I = \emptyset$ \\
reflexivity assertion$^*$   & $\mathrm{Ref}(R)$ & $\{ \Pair{x}{x}|\, x \in \Delta^\I \} \subseteq R^\I$ \\
irreflexivity assertion$^*$   & $\mathrm{Irr}(R)$ & $R^\I \cap \{ \Pair{x}{x}|\, x \in \Delta^\I \} = \emptyset$\\
symmetry assertion   & $\mathrm{Sym}(R)$ & $\Pair{x}{y} \in R^\I \Rightarrow \Pair{y}{x} \in R^\I$ \\
asymmetry assertion$^*$   & $\mathrm{Asym}(R)$ & $\Pair{x}{y} \in R^\I \Rightarrow \Pair{y}{x} \notin R^\I$ \\
transitivity assertion   & $\mathrm{Tra}(R)$ & 
$\{\Pair{x}{y}\!,\Pair{y}{z}\} \subseteq R^\I \Rightarrow \Pair{x}{z}\in R^\I$\\[1ex]
concept assertion 	& $C(a)$	& $a^\I \in C^\I$ \\
role assertion 		& $R(a,b)$	& $\Pair{a^\I}{b^\I} \in R^\I$ \\
negated role assertion & $\neg R(a,b)$	& $\Pair{a^\I}{b^\I} \notin R^\I$\\
equality assertion & $a = b$	& $a^\I = b^\I$\\
inequality assertion & $a \neq b$	& $a^\I \neq b^\I$\\[.5ex]
\hline
\end{tabular}
\end{tabular}
\end{center}
\vspace{-\baselineskip}
\end{table}

\section{Reasoning and complexity: more details}

In what follows, we assume the setting of  \cite{BozzatoES:18} for the
complexity analysis.

\begin{proposition}
Deciding whether a CAS-interpretation $\IC_{\CAS}$ of a  sCKR $\CKB$
is a CKR-model is coNP-complete. 
\end{proposition}

Informally, $\IC_{\CAS}$  can be refuted if it is not a justified
CAS-model of $\CKB$, which  can be checked in polynomial time using
the techniques in \cite{BozzatoES:18}, or some preferred model 
$\IC'_{\CAS}$ exists; the latter can be guessed and checked in
polynomial time. 
The coNP-hardness is shown, already under data complexity, by a
reduction from a restricted version of UNSAT. We shall discuss in the
context of $\smlc$-entailment under data complexity below. 


\begin{theorem}
Given a ranked sCKR $\CKB$, a context name $\smlc$ and an axiom $\alpha$,  deciding
whether $\CKB\models \smlc\,{:}\,\alpha$  is $\Delta^p_2$-complete for profile-based preference.
\end{theorem}
\begin{proof}[Sketch]
%
For profile-based comparison, we can compute the lexicographic 
maximum profile $p^*$ of a CKR-model by
extending a partial profile $(l^*_n,l^*_{n-1},\ldots,l^*_i)$,
$i=n,n-1,\ldots,0$ using an NP oracle in polynomial time;
asking for each possible value $v$ whether $l_j=v$ is possible.
We then can check with the NP oracle whether every justified CAS-model $\IC_{\CAS}$ having
this profile fulfills $\IC_{\CAS} \not\models \smlc\,{:}\,\alpha$.

The $\Delta^p_2$-hardness is shown by a reduction from
deciding the last bit $x_n$ of the lexicographic maximum 
satisfying assignment of a SAT instance $E = \bigwedge_{i=1}^m \gamma_i$ over propositional atoms $X=\{x_1,\ldots,x_n\}$.

Without loss of generality,  $E$ is a 3SAT instance (with duplicate
literals allowed) and each clause $\gamma_i$ in $E$ is either positive or negative.

Then we construct $\CKB$ as follows. Let 
$V_i$, $i=1,\ldots,n$
 and $F,T,A$ be concepts, $P_1,P_2,P_3,N_1,N_2,N_3$ be
roles, and $x_1,\ldots,x_n,$ $c_1,\ldots,c_m$ be individual constants.
We use totally ordered contexts $\smlc_0 < \smlc_1< \cdots <
 \smlc_{n+1}$. The  knowledge bases of the contexts contain the following axioms

\begin{itemize}
\itemsep=0pt
 \item the knowledge base of $\smlc_{n+1}$ contains the  defeasible axioms
$D(V_i \isa F)$ for all $i=1,\ldots,n$

 \item the knowledge base of $\smlc_i$, $i=1,\ldots,n$ contains
 the defeasible axiom $D(V_i \isa T)$
    
 \item the knowledge base of  $\smlc_0$ that contains the inclusion axioms: 

$T\sqcap F \isa \bot$,\quad $T \isa A$,\quad $F \isa A$,\quad $\bigsqcap_{j=1}^3\exists N_j.(T\sqcap A) \isa \bot$,\quad and $\bigsqcap_{j=1}^3\exists P_j.(F\sqcap A) \isa \bot$,

 \noindent
 and the assertions 

 \item $V_h(x_h)$, $h=1,\ldots,n$, and 
 \item $P_j(c_i,x_{i_j})$ for $i=1,\ldots,m$ and $j=1,2,3$ such that
     the clause $\gamma_i$ is of form $x_{i_1}\lor x_{i_2}\lor x_{i_3}$,
 \item $N_j(c_i,x_{i_j})$ for $i=1,\ldots,m$ and $j=1,2,3$ such that
     the clause $\gamma_i$ is of form $\neg x_{i_1}\lor \neg x_{i_2}\lor \neg x_{i_3}$.
\end{itemize}

Intuitively, we must at context $\smlc_0$ make for each $x_h$ an exception to
either $V_i\isa F$ or $V_i\isa T$; the respective single minimal clashing set is
$\{ V_i(x_h), \neg F(x_h) \}$ resp.\ $\{ V_i(x_h), \neg T(x_h) \}$. 

One can show that the justified CAS-models $\IC_{\CAS}$ of the CKR correspond 1-1 to the
satisfying assignments $\sigma$ of $E$.  Furthermore, under
profile-based preference, keeping $V_i \sqsubseteq T$ is preferred
over keeping $V_i \sqsubseteq F$, and thus by the context ordering the
lexicographic maximum $\sigma^*$ that satisfies $E$ will be reflected in
every non-preferred model. Consequently, 
$\CKB\models \smlc_0:T(x_n)$ holds iff $\sigma(x_n)=$ true.
\end{proof}

\begin{theorem}
Deciding where an sCKR $\CKB$ entails a Boolean CQ $\gamma$ is $\Pi^p_2$-complete
for profile-based preference.
\end{theorem}
\begin{proof}[Sketch]
Similarly as for $\smlc$-entailment, a CKR-model $\IC_{\CAS}$ that
does not entail $\gamma$ can be guessed and checked with the help of
an NP oracle (ask whether no CKR-model $\IC'_{\CAS}$ of $\CKB$ exists
that is preferred to $\IC_{\CAS}$ and
whether $\gamma$ is entailed in $\IC_{\CAS}$); note that the profiles
of interpretations are easy to calculate.
The $\Pi^p_2$-hardness is inherited from ordinary CKR.

\end{proof}

\subsection{Data complexity}

Concerning the data complexity, i.e., the CKR $\CKB$ is fixed and only
the assertions in the knowledge modules vary,

\begin{proposition}
Deciding whether a given CAS-interpretation $\IC_{\CAS}$ of a  sCKR $\CKB$
is a CKR-model is coNP-complete under data complexity. 
\end{proposition}
\begin{proof}[Sketch]
The membership is inherited from the general case.  
The hardness part follows from the particular reduction of deciding ODD SAT
to $\smlc$-entailment under data complexity, which amounts for 
particular inputs to a reduction from a variant of UNSAT,
and will be discussed in
this context.
\end{proof}

\begin{theorem}
Deciding
whether $\CKB\models \smlc\,{:}\,\alpha$  is $\Delta^p_2[O(\log n)]$-complete for profile-based preference.
under data complexity.
\end{theorem}
\begin{proof}[Sketch]
The membership in   $\Delta^p_2[O(\log n)]$ is
established by exploiting that\linebreak 
$\Delta^p_2[O(\log n)] = {\rm
P^{NP}}|_{\|[k]}$ (cf.\ \cite{eite-gott-97}): we can compute, with
parallel NP oracle queries, in a constant number of rounds the optimal
profile $p^*=(l^*_n,\ldots,l^*_n)$ of any clashing assumption $\chi$ of a CKR-model, 
as $n$ is constant: in each round, we extend the partial profile
$(l^*_n,\ldots,l^*_{j+1})$ with $l^*_j$, asking for each possible
value $v$ whether $l_j=v$ is possible.  In a last round, we can then
decide with a single oracle call  $\CKB\models \smlc\,{:}\,\alpha$
based on $p^*$.

The  $\Delta^p_2[O(\log n)]$-hardness is shown by a reduction from
deciding whether among given 3SAT instances $E_1,\ldots,E_l$, $l\geq
1$ on disjoint atoms, where duplicate literals in clauses are allowed,
and an odd number of $E_k$ is satisfied by some assignment that does not set
all atoms in $E_k$ to false. 
The $\Delta^p_2[O(\log n)]$-completeness of this problem, which we
refer to as ODD SAT  follows from \cite{wagn-90}.
Without loss of generality, we 
may assume that $E_k$ is only satisfiable if $E_{i-1}$ is, that $l$ is
even, that all $E_k$ have the same number of variables, that the clauses
in them are monotone, and that each satisfying assignment of $E_k = E_k(x^i_1,\ldots,x^i_n)$
sets either all
atoms to false or otherwise $x^i_1$ to true.

Then we construct $\CKB$ similar as for ${\rm P^{NP}}$-hardness  follows. Let 
$F,T,A,V,Y,O$ be concepts, $P_1,P_2,P_3,N_1,N_2,N_3,C,R$ be
roles, and $a$ and $x^k_1,\ldots,x^k_n,$ $c^k_1,\ldots,c^k_{m_i}$ be
individual constants for the variables and clauses in $E_i$, respectively.
We use totally ordered contexts $\smlc_0 < \smlc_1 < \smlc_2$. The  knowledge bases of the contexts contain the following axioms

\begin{itemize}
\itemsep=0pt
 \item the knowledge base of $\smlc_2$ contains the  defeasible axioms
    $D(V \isa F)$.

 \item the knowledge base of $\smlc_1$ contains the defeasible axiom $D(V \isa T)$
    
 \item the knowledge base of  $\smlc_0$ that contains the inclusion axioms: 

$T\sqcap F \isa \bot$,\quad $T \isa A$,\quad $F \isa A$,\quad
    $\bigsqcap_{j=1}^3\exists N_j.(T\sqcap A) \isa \bot$, \quad
$\bigsqcap_{j=1}^3\exists P_j.(F\sqcap A) \isa \bot$, 
\quad $T \sqcap \exists C.F \sqsubseteq Y$, \quad $O \sqsubseteq \exists R.Y$ 

\noindent
 and the assertions 

 \item $V(x^k_j)$, for all $i$ and $j$,
 \item $P_j(c^k_i,x^k_{i_j})$ for $i=1,\ldots,m$ and $j=1,2,3$ such that
     the clause $\gamma^k_i$  of $E_k$ is of form $x^k_{i_1}\lor x^k_{i_2}\lor x^k_{i_3}$,
 \item $N_j(c^k_i,x^k_{i_j})$ for $i=1,\ldots,m$ and $j=1,2,3$ such that
     the clause $\gamma^k_i$ is of form $\neg x^k_{i_1}\lor \neg
    x^k_{i_2}\lor \neg x^k_{i_3}$,
 \item $C(x^{k+1}_1,x^{k+2})$, $R(a,x^{2k+1}_1)$ for $k=0,2,\ldots,l-2$. 
\end{itemize}
Intuitively, we must at context $\smlc_0$ make for each $x^k_j$ an exception to
either $V \isa F$ or $V \isa T$; the respective single minimal clashing set is
$\{ V(x^k_i), \neg F(x^k_i) \}$ resp.\ $\{ V(x^k_i), \neg T(x^k_i) \}$. 

One can show that the (preference-less) CKR- $\IC_{\CAS}$ of the CKR correspond 1-1 to the
combinations of satisfying assignments $\sigma_1,\ldots,\sigma_l$ of
$E_1,\ldots,E_l$, respectively.  Furthermore, under
profile-based preference, keeping $V \sqsubseteq T$ is preferred
over keeping $V \sqsubseteq F$, and thus by the context ordering for
each $E^k$ an assignment $\sigma_k$ that sets $x^k_1$ to true.

In case an odd number of $x^k_1$ is set to true, for some
$x^{2k'+1}_1$  from the assertion $C(x^{2k'+1}_1,x^{2k'+2})$ and 
$T \sqcap \exists C.F \sqsubseteq Y$, one can derive $Y(x^{2k'+1})$,
and then from $R(a,x^{2k+1}_1)$ and the axiom $O \sqsubseteq \exists R.Y$ 
that $O(a)$ holds. On the other hand, $O(a)$ can not be derived if 
an even number of $x^k_1$ is set to true.

Consequently, $\CKB \models \smlc_0{:}O(a)$ holds iff the instance of
ODD SAT is a yes-instance. This  shows $\Delta^p_2[O(\log n)]$-hardness.
\end{proof}

\medskip

We remark that the reduction in the proof establishes coNP-hardness of
model checking under data complexity: if we consider $l=2$ and
an $E_2$ that is satisfied only if all atoms are set to false, then
for the clashing assumption $\chi$ consisting of  $\stru{V\sqsubseteq
T, x^k_i}$  for all atoms $x^k_j$ gives rise to a (canonical)
CKR-model $\IC_{\CAS}$ of $\CKB$ that can be constructed in polynomial
time, and moreover  $\IC_{\CAS}$ is preferred iff $E_1$ is
unsatisfiable; this shows coNP-hardness (a simpler, direct
construction for $E_1$ is clearly possible).

\medskip

That CQ entailment remains $\Pi^p_2$-complete under data complexity is
a simple consequence that membership in $\Pi^p_2$ holds for the
general case, and that the inherited $\Pi^p_2$-hardness of
CQ-answering for ordinary CKR knowledge bases (without context hierarchies) holds under data complexity.

\subsection{Complexity of $\smlc$-Entailment under local preference}

As for local preference at a context $\smlc$, let for any context
$\smlc'$ above $\smlc$ denote $X_{\smlc}(\smlc')$ the set of 
all clashing assumptions  $\stru{\alpha, \ee}$ for defeasible axioms
at $\smlc'$ made at $\smlc$ in some CKR-model of $\CKB$.

Call a context $\smlc'$ a connector for $\smlc$, if it  
directly covers
$\smlc$ and for every $\smlc''$ and $\smlc'''$, if 
and $\smlc''$ covers $\smlc'''$, then 
$\smlc'''$ covers $\smlc'$ (i.e., every path in 
the covers-graph from a node $\smlc''$ above $\smlc'$ to $\smlc$ must
pass through $\smlc'$).

Consider the following property of the local preference $>$:

\begin{itemize}
\item[](CP)  If $\smlc'$ is a connector for $\smlc$ and
 (i) $X_{\smlc}(\smlc')\subseteq \chi_1(\smlc)$,  
 (ii) $X_{\smlc}(\smlc')\not\subseteq \chi_2(\smlc)$, and 
 (iii)  $\chi_1(\smlc'') = \chi_2(\smlc'')$, for each $\smlc'' > \smlc$
  such that $\smlc'' \not> \smlc'$, then $\chi_2(c) > \chi_1(\smlc)$.
\end{itemize}
That is, the worst possible overriding at a connector for $\smlc$ is
always less preferred, if the clashing assumptions agree on the
contexts that are not above $\smlc'$ or reachable from some such context. This condition seems to be
plausible for local preference.

Let the {\em global preference}\/ on CAS-models 
$\IC^1_{\CAS} = \stru{\IC, \casmap_1}$, 
$\IC^2_{\CAS} = \stru{\IC, \casmap_2}$ 
induced by a local preference $\chi_1(\smlc) > \chi_2(\smlc)$ on clashing
assumptions of contexts $\smlc$ be 
as follows: $\IC^1_{\CAS}$ is preferred
to $\IC^2_{\CAS}$, if
there exists some $\mlc \in \N$ s.t. $\casmap_1(\mlc) > \casmap_2(\mlc)$ and
for no context $\mlc' \neq \mlc \in \N$ it holds that $\casmap_1(\mlc') < \casmap_2(\mlc')$.

\begin{theorem}
Suppose  $\CKB$ is a sCKR with global preference induced by a local 
preference $>$ that is polynomial-time decidable and satisfies (CP).
Then $\smlc$-entailment $\CKB\models \smlc:\alpha$ is
$\Pi^p_2$-complete. Furthermore, the $\Pi^p_2$-hardness even holds for ranked
hierarchies with three levels.
\end{theorem}
\begin{proof}[Sketch]
The membership in $\Pi^p_2$ follows by a guess an check argument,
as we can guess a CKR-model
$\IC_{\CAS}$ of $\CKB$ such that (i)  $\IC_{\CAS} \not\models
\smlc\,{:}\,\alpha$  and (ii) no CKR-model
$\IC'_{\CAS}$ of $\CKB$ exists such that $\IC'_{\CAS}> \IC_{\CAS}$. As
local model checking in absence of preferences is polynomial, and
local preference is polynomial decidable, $\CKB\models \smlc:\alpha$
is decidable in non-deterministic polynomial time with an NP oracle,
and thus in $\Sigma^p_2$, which implies the result.

The $\Pi^p_2$-hardness of $\smlc$-entailment under the given assertion
can be shown by a
reduction from  evaluating a QBF of the form $\forall X\exists Y
\neq \mu E(X,Y)$, where w.l.o.g. 
$E = \bigwedge_{i=1}^m \gamma_i$  is a monotone 3CNF (each clause is
either positive or negative and has size 3, with duplicate literals
allowed), and $Y = X'Y'$, where $E$ contains the clauses
$x_i \lor x'_i$, $\neg x_i \lor \neg x'_i$ (i.e.,  $x_i
\leftrightarrow \neg x'_i$) for each $x_i \in X$, and $\mu$ is a
particular assignment to $Y$ such that $\forall X E(X,\mu(Y))$
evaluates to true.

We construct a CKR $\CKB$ as follows.
We use contexts $\smlc_0$ and 
$\smlc_p$, $\smlc_{\hat{p}}$,  for all $p \in X\cup X'$ and
$\smlc_{Y'}$, $\smlc_{\neg Y'}$. 
The context ordering is

\begin{itemize}
            \item   $\smlc_0 < \smlc_p < \smlc_{p'}$,  for all $p\in
                        X\ \cup X'$, 
             \item $\smlc_0 <\smlc_{Y'} < \smlc_{\neg Y'}$.
\end{itemize}

 Let 
$V_i$, $i=1,\ldots,2|X|+|Y|$ and $F,T,A$ be concepts, $P_1,P_2,P_3,N_1,N_2,N_3$ be
roles, and $p \in X \cup X'$, $y_j\in Y$ and $c_1,\ldots,c_m$ be individual constants.

The  knowledge bases of the contexts contain the following axioms

\begin{itemize}
\itemsep=0pt
 \item the knowledge base of $\smlc_{p}$ (resp.,  $\smlc_{p'}$) contains the  defeasible axiom
$D(V_i \isa T)$ (resp., $D(V_i \isa p)$)  if $p=x_i$ resp.\ if
                        $p=x'_j$ and $i=2j$;

 \item the knowledge base of $\smlc_{Y'}$ 
          the defeasible axioms $D(V_{2|X|+j} \isa T)$ if $\mu(y_j) =$
          true, and $D(V_{2|X|+j} \isa F)$ otherwise,
           for all $j=1,\ldots,|Y'|$
    
 \item the knowledge base of $\smlc_{\neg Y'}$ 
          the defeasible axioms $D(V_{2|X|+j} \isa F)$ if $\mu(y_j) =$
          true, and $D(V_{2|X|+j} \isa T)$ otherwise,
           for all $j=1,\ldots,|Y'|$
    
 \item the knowledge base of  $\smlc_0$ that contains the inclusion axioms: 

$T\sqcap F \isa \bot$,\quad $T \isa A$,\quad $F \isa A$,\quad $\bigsqcap_{j=1}^3\exists N_j.(T\sqcap A) \isa \bot$, \quad and\ \ $\bigsqcap_{j=1}^3\exists P_j.(F\sqcap A) \isa \bot$,

 \noindent
 and the assertions 

 \item $V_h(x_h)$,  $V_{n+h}(x'_h)$, $h=1,\ldots,n$, and  
$V_{2|X|+j}(y_j)$  $j=1,\ldots,|Y'|$, 

 \item $P_j(c_i,v_{i_j})$ for $i=1,\ldots,m$ and $j=1,2,3$ such that
     the clause $\gamma_i$ is of form $v_{i_1}\lor v_{i_2}\lor v_{i_3}$,
 \item $N_j(c_i,v_{i_j})$ for $i=1,\ldots,m$ and $j=1,2,3$ such that
     the clause $\gamma_i$ is of form $\neg v_{i_1}\lor \neg v_{i_2}\lor \neg v_{i_3}$.
\end{itemize}

Informally,  either $V_p \isa T$ or $V_p \isa F$ is overridden in each
CKR-model for each $p=x_i$ resp. $p=x'_i$, which correspond
to truth assignments to $X$ and $X'$; as $x_i \leftrightarrow \neg
x'_i$, justified CAS-models are only comparable for preference if the
correspond to the same assignment. On the other hand, by the
assumption of $\mu$ for $E$, we have some CKR-model $\IC_{\CAS}$ in which all overriding of
$V_j \isa T$ or $V_i \isa F$  for atoms $y_j$ happens for 
the axioms at $\smlc_{Y'}$.  That is, 
the clashing assumption of $\IC_{\CAS}$ includes the 
set $X_{\smlc}(\smlc')$ of clashings assumptions defined above 
for $\smlc=\smlc_0$ and $\smlc'=\smlc_{Y'}$.

We may assume that if for a given assignment $\sigma$ to $X\cup X'$
some other assignment $\mu'$ to $Y'$ exists that makes $E(\sigma(X\cup
X'),Y')$ true, then some fixed variable $y_i$ must in $\mu'$ have,
regardless of $\sigma$, a different value than in $\mu$.

Note that $\smlc_{y'}$ is a connector of $\smlc$ (as is $\smlc_p$ for
every atom $p$). Under the assumption that the local preference $>$ satisfies the
property (CP), it follows that the corresponding CKR-model $\IC'_{\CAS}$ will then be
preferred to the model $\IC_{\CAS}$ for $\sigma,\mu$.

Consequently,
$\CKB\models \smlc_0\,{:}\,T(y_i))$  respectively
$\CKB\models \smlc_0\,{:}\,F(y_i))$  holds iff the formula
$\forall X\exists Y \neq \mu E(X,Y)$ evaluates to true. This shows
$\Pi^p_2$-hardness of $\smlc$-entailment under a global preference
induced by any polynomial-time decidable local preference
$\chi_2(\smlc)> \chi_1(\smlc)$ that
satisfies (CP). A particular such preference is e.g.\ profile-based
based preference a the local level.

We note that the contexts $\smlc_{Y'}$ and $\smlc_{\neg Y'}$ can be
replaced by copies $\smlc_{y'_j}$ and $\smlc_{\neg y'_j}$, for all
$y'_j \in Y'$; each  $\smlc_{y'_j}$ is a connector. Thus, the
$\Pi^p_2$-hardness carries over to the case of a ranked hierarchy with
three levels. In case of two levels, no context-sensitive overriding
is possible and the setting is subsumed by the one of ordinary CKR,
for which $\smlc$-entailment is coNP-complete.
\end{proof}

\section{Translation rule set tables}

Rule sets for the proposed translation are shown in the tables in following pages.\linebreak
$\SROIQrl$ input and deduction rules are presented in Table~\ref{tab:rl-rules-tgl}.
Table~\ref{tab:global-local-rules-tgl} lists global and local translations
and output rules.
Table~\ref{tab:input-default-tgl} shows input rules for defeasible axioms.
Overriding rules are shown in Table~\ref{tab:ovr-rules-tgl},
defeasible inheritance rules are reported in Table~\ref{tab:inheritance-rules-tgl}
and test rules are shown in Table~\ref{tab:test-rules-tgl}.
Finally, the newly introduced rules and constraints for overriding level preference 
are shown in Table~\ref{tab:pref-rules}.

\begin{table}[tp]%
\caption{$\SROIQrl$ input and deduction rules}
\vspace{2ex}
\hrule\mbox{}\\ 
\textbf{$\SROIQrl$ input translation $I_{rl}(S,c)$}\\[.5ex]
\scalebox{.9}{
\small
$\begin{array}[t]{l@{\ \ }l}               
\mbox{(irl-nom)} 
& a \in \NI \mapsto \{\nom(a,c)\}\\
\mbox{(irl-cls)} 
& A \in \NC \mapsto \{\cls(A,c)\}\\
\mbox{(irl-rol)} 
& R \in \NR \mapsto \{\rol(R,c)\}\\[1ex]

\mbox{(irl-inst1)} 
& A(a) \mapsto \{\insta(a,A,c,\ml{main})\} \\
\mbox{(irl-inst2)} 
& \non A(a) \mapsto \{\ninsta(a,A,c)\} \\
\mbox{(irl-triple)} 
& R(a,b) \mapsto \{\triplea(a,R,b,c,\ml{main})\} \\
\mbox{(irl-ntriple)} & \non R(a,b) \mapsto \{\ntriplea(a,R,b,c)\} \\
\mbox{(irl-eq)} & a = b \mapsto \{\peq(a,b,c,\ml{main})\} \\
\mbox{(irl-neq)} & a \neq b \mapsto \emptyset \\
\mbox{(irl-inst3)} 
& \{a\} \subs B \mapsto \{\insta(a,B,c, \ml{main})\} \\
\mbox{(irl-subc)} 
& A \subs B \mapsto \{\subClass(A,B,c)\} \\
\mbox{(irl-top)} & \top(a) \mapsto \{\insta(a,\ptop,c)\} \\
\mbox{(irl-bot)} & \bot(a) \mapsto \{\insta(a,\pbot,c)\}

\end{array}$
\;
$\begin{array}[t]{l@{\ \ }l}               
\mbox{(irl-subcnj)} 
& A_1 \sqcap A_2 \subs B \mapsto \{\subConj(A_1,A_2,B,c)\} \\
\mbox{(irl-subex)} 
& \exists R.A \subs B \mapsto \{\subEx(R,A,B,c)\} \\[1ex]

\mbox{(irl-supex)} &  A \subs \exists R.\{a\} \mapsto \{\supEx(A,R,a,c)\}\\
\mbox{(irl-forall)} &  A \subs \forall R.B \mapsto \{\supForall(A,R,B,c)\} \\
\mbox{(irl-leqone)} &  A \subs {\leqslant} 1 R.\top \mapsto \{\supLeqOne(A,R,c)\} \\[1ex]
        
\mbox{(irl-subr)} 
& R \subs S \mapsto \{\subRole(R,S,c)\}\\
\mbox{(irl-subrc)} 
& R {\circ} S \subs T \mapsto \{\subRChain(R,S,T,c)\}\\
\mbox{(irl-dis)} & \mathrm{Dis}(R,S)  \mapsto \{\pDis(R,S,c)\}\\
\mbox{(irl-inv)} & \mathrm{Inv}(R,S) \mapsto \{\pInv(R,S,c)\}\\
\mbox{(irl-irr)} & \mathrm{Irr}(R) \mapsto \{\pIrr(R,c)\}\\
  \end{array}$}\\[.3ex]
  
\textbf{$\SROIQrl$ deduction rules $P_{rl}$}\\[.5ex]
\scalebox{.9}{
\small
$\begin{array}{l@{\;}r@{\ }r@{\ }l@{}}
	 \mbox{(prl-instd)} & \instd(x,z,c,t) & \rif & \insta(x,z,c,t).\\
	 \mbox{(prl-tripled)} & \tripled(x,r,y,c,t) & \rif & \triplea(x,r,y,c,t).\\[1ex]

	 \mbox{(prl-ninstd)} & \unsat(t) & \rif & \ninsta(x,z,c), \instd(x,z,c,t).\\
	 \mbox{(prl-ntripled)} & \unsat(t) & \rif & \ntriplea(x,r,y,c), \tripled(x,r,y,c,t).\\[1ex]

   \mbox{(prl-eq)} & \unsat(t) & \rif & \peq(x,y,c,t).\\
   \mbox{(prl-top)} & \instd(x,\ptop,c,\ml{main}) & \rif & \nom(x,c).\\	
   \mbox{(prl-bot)} & \unsat(t) & \rif & \instd(x,\pbot,c,t).\\[1ex]
		
   \mbox{(prl-subc)} 
   &      \instd(x,z,c,t) & \rif & \subClass(y,z,c), \instd(x,y,c,t). \\
   \mbox{(prl-subcnj}) 
   & \instd(x,z,c,t) & \rif & \subConj(y_1,y_2,z,c), \instd(x,y_1,c,t),\instd(x,y_2,c,t). \\
   \mbox{(prl-subex)} 
   & \instd(x,z,c,t) & \rif & \subEx(v,y,z,c), \tripled(x,v,x',c,t),\instd(x',y,c,t). \\
   \mbox{(prl-supex)} 
   & \tripled(x,r,x',c,t) & \rif & \supEx(y,r,x',c), \instd(x,y,c,t). \\
   \mbox{(prl-supforall)} 
   & \instd(y,z',c,t) & \rif & \supForall(z,r,z',c), \instd(x,z,c,t), \tripled(x,r,y,c,t).\\
   \mbox{(prl-leqone)} & \unsat(t) & \rif & \supLeqOne(z,r,c), \instd(x,z,c,t),\\
           & & & \tripled(x,r,x_1,c,t),\tripled(x,r,x_2,c,t).\\[1ex]					

   \mbox{(prl-subr)} 
   & \tripled(x,w,x',c,t) & \rif & \subRole(v,w,c), \tripled(x,v,x',c,t). \\
   \mbox{(prl-subrc)} 
   & \tripled(x,w,z,c,t) & \rif & \subRChain(u,v,w,c), \tripled(x,u,y,c,t), \tripled(y,v,z,c,t).\\[1ex]

   \mbox{(prl-dis)} & \unsat(t) & \rif & \pDis(u,v,c), \tripled(x,u,y,c,t), \tripled(x,v,y,c,t).\\	
   \mbox{(prl-inv1)} & \tripled(y,v,x,c,t) & \rif & \pInv(u,v,c), \tripled(x,u,y,c,t). \\
   \mbox{(prl-inv2)} & \tripled(y,u,x,c,t) & \rif & \pInv(u,v,c), \tripled(x,v,y,c,t). \\
   \mbox{(prl-irr)} & \unsat(t) & \rif & \pIrr(u,c),\tripled(x,u,x,c,t).\\[1ex]
   \mbox{(prl-sat)} &   & \rif & \unsat(\main).
  \end{array}$}\\[.5ex]
\hrule\mbox{}
\label{tab:rl-rules-tgl}
\end{table}

\begin{table}[tp]%
\caption{Global, local and output rules}
\label{tab:global-local-rules-tgl}

\medskip

\hrule\mbox{}\\
\textbf{Global input rules $I_{glob}(\Cl)$}\\[.5ex]
\scalebox{1}{
\small
$\begin{array}{l@{\ \ }l}
 \mbox{(igl-covers)} & \mlc_1 \prec \mlc_2 \mapsto \{\pprec(\mlc_1, \mlc_2)\}\\
 \mbox{(igl-level)} & l(\mlc_1) = n \mapsto \{\level(\mlc_1, n + 1)\}\\[1ex]
\end{array}$}

\textbf{Local input rules $I_{loc}(\KB_\mlm, \mlc)$}\\[.5ex]
\scalebox{1}{
\small
$\begin{array}{l@{\ \ }l}
\mbox{(ilc-subevalat)} & \eval(A, \ml{c}_1) \subs B \mapsto \{\subEval(A, \ml{c}_1, B, \mlc)\} \\
\mbox{(ilc-subevalr)} & \eval(R, \ml{c}_1) \subs T \mapsto \{\subEvalR(R, \ml{c}_1, T, \mlc)\}\\[1ex]
\end{array}$}

\textbf{Local deduction rules $P_{loc}$}\\[.5ex]
\scalebox{.95}{
\small
$\begin{array}{l@{\ \ }l}               
\mbox{(plc-subevalat)} & 
\instd(x, b, c, t) \rif \subEval(a, c_1, b, c), \instd(x, a, c_1, t).\\
\mbox{(plc-subevalr)} & 
\tripled(x, s, y, c, t) \rif \subEvalR(r, c_1, s, c), \tripled(x, r, y, c_1, t).\\[1ex]

\mbox{(plc-subevalatp)} & 
\instd(x, b, c, t) \rif \subEval(a, c_1, b, c_2), \instd(x, a, c_1, t), \pprec(c, c_2).\\
\mbox{(plc-subevalrp)} & 
\tripled(x, s, y, c, t) \rif \subEvalR(r, c_1, s, c_2), \tripled(x, r, y, c_1, t), \pprec(c, c_2).\\[1ex]

\end{array}$}

\textbf{Output translation $O(\alpha,\mlc)$}\\[.5ex]
\scalebox{1}{
\small
$\begin{array}{l@{\ \ }l}
\mbox{(o-concept)} & A(a) \mapsto \{\instd(a,A,\mlc,\main)\} \\
\mbox{(o-role)} & R(a,b) \mapsto \{\tripled(a,R,b,\mlc,\main)\} \\[1ex]
\end{array}$}
\hrule
\end{table}

\begin{table}[tp]%
\caption{Input rules $I_{\default}(S,c)$ for defeasible axioms}
\label{tab:input-default-tgl} 

\medskip

\hrule\mbox{}\\[1ex]
{
\small
\scalebox{1}{
$\begin{array}{@{}l@{~}r@{~}l@{}}
 \mbox{(id-inst)} & \default(A(a))  & \mapsto \{\, \definst(A,a,c).\,\} \\
 \mbox{(id-triple)} & \default(R(a,b))  & \mapsto \{\, \deftriple(R,a,b,c).\,\} \\ 
 \mbox{(id-ninst)} & \default(\non A(a))  & \mapsto \{\, \defninst(A,a,c).\,\} \\
 \mbox{(id-ntriple)} & \default(\non R(a,b))  & \mapsto \{\, \defntriple(R,a,b,c).\,\} \\ 

 \mbox{(id-subc)} & \default(A \subs B)  & \mapsto \{\, \defsubs(A,B,c).\,\} \\  
 \mbox{(id-subcnj)} & \default(A_1 \sqcap A_2 \subs B)  & \mapsto \{\, \defsubcnj(A_1, A_2, B,c).\,\} \\   
 \mbox{(id-subex)} & \default(\exists R.A \subs B)  & \mapsto \{\, \defsubex(R, A, B,c).\,\} \\    
 \mbox{(id-supex)} & \default(A \subs \exists R.\{a\})  & \mapsto \{\, \defsupex(A, R, a,c).\,\} \\   
 \mbox{(id-forall)} & \default(A \subs \forall R.B)  & \mapsto \{\, \defsupforall(A, R, B,c).\,\} \\    
 \mbox{(id-leqone)} & \default(A \subs {\leqslant} 1 R.\top)  & \mapsto \{\, \defsupleqone(A, R,c).\,\} \\[2ex]
 \mbox{(id-subr)} & \default(R \subs S)  & \mapsto \{\, \defsubr(R, S,c).\,\} \\    
 \mbox{(id-subrc)} & \default(R \circ S \subs T)  & \mapsto \{\, \defsubrc(A_1, A_2, B,c).\,\} \\    
 \mbox{(id-dis)} & \default(\mathrm{Dis}(R,S))  & \mapsto \{\, \defdis(R, S,c).\,\} \\   
 \mbox{(id-inv)} & \default(\mathrm{Inv}(R,S))  & \mapsto \{\, \definv(R,S,c).\,\} \\   
 \mbox{(id-irr)} & \default(\mathrm{Irr}(R))  & \mapsto \{\, \defirr(R,c).\,\} \\    
\end{array}$}}\\[1ex]
\hrule
\end{table}

\begin{table}[tp]%
\caption{Deduction rules $P_{\default}$ for defeasible axioms: overriding rules}
\label{tab:ovr-rules-tgl}

\medskip

\hrule\mbox{}\\[1ex]
\scalebox{.78}{
$\begin{array}{l@{\ \ }r@{\ \ }l}
 \mbox{(ovr-inst)} & 
 \ovr(\insta,x,y,c_1,c)  \rif & \definst(x,y,c_1), \pprec(c,c_1), \naf \testf(\nlit(x,y,c)). \\[0.5ex]
 \mbox{(ovr-triple)} & 
 \ovr(\triplea,x,r,y,c_1,c) \rif & \deftriple(x,r,y,c_1), \pprec(c,c_1), \naf \testf(\nrel(x,r,y,c)). \\[0.5ex]
 \mbox{(ovr-ninst)} & 
 \ovr(\ninsta,x,y,c_1,c)  \rif & \defninst(x,y,c_1), \pprec(c,c_1), \instd(x,z,c,\ml{main}). \\[0.5ex]
 \mbox{(ovr-ntriple)} & 
 \ovr(\ntriplea,x,r,y,c_1,c) \rif & \defntriple(x,r,y,c_1), \pprec(c,c_1), \tripled(x,r,y,c,\ml{main}). \\[0.5ex]
 \mbox{(ovr-subc)} & 
 \ovr(\subClass,x,y,z,c_1,c) \rif &
 \defsubs(y,z,c_1), \pprec(c,c_1), \instd(x,y,c,\ml{main}),\\
  & & \naf \testf(\nlit(x,z,c)). \\[0.5ex]
  \mbox{(ovr-cnj)} & 
  \ovr(\subConj,x,y_1,y_2,z,c_1,c) \rif &  \defsubcnj(y_1,y_2,z,c_1), \pprec(c,c_1), \instd(x,y_1,c,\ml{main}), \\
                          &&  \instd(x,y_2,c,\ml{main}),  \naf \testf(\nlit(x,z,c)).\\[0.5ex]
  \mbox{(ovr-subex)} & 
  \ovr(\subEx,x,r,y,z,c_1,c) \rif &  \defsubex(r,y,z,c_1), \pprec(c,c_1), \tripled(x,r,w,c,\ml{main}), \\
                && \instd(w,y,c,\ml{main}),   \naf \testf(\nlit(x,z,c)).\\[0.5ex]
  \mbox{(ovr-supex)} & 
  \ovr(\supEx,x,y,r,w,c_1,c) \rif &  \defsupex(y,r,w,c_1), \pprec(c,c_1), \\
  &&  \instd(x,y,c,\ml{main}), \naf \testf(\nrel(x,r,w,c)). \\[0.5ex]
  \mbox{(ovr-forall)} & 
  \ovr(\supForall,x,y,z,r,w,c_1,c) \rif & \defsupforall(z,r,w,c_1),
                \pprec(c,c_1), \instd(x,z,c,\ml{main}),\\
                 && \tripled(x,r,y,c,\ml{main}),  \naf \testf(\nlit(y,w,c)). \\[0.5ex]
  \mbox{(ovr-leqone)} & 
  \ovr(\supLeqOne,x,x_1,x_2,z,r,c_1,c) \rif &
  \defsupleqone(z,r,c_1), \pprec(c,c_1), \instd(x,z,c,\ml{main}), \\
  && \tripled(x,r,x_1,c,\ml{main}),\tripled(x,r,x_2,c,\ml{main}),\\[0.5ex]
  \mbox{(ovr-subr)}  & 
  \ovr(\subRole,x,y,r,s,c_1,c) \rif &
  \defsubr(r,s,c_1), \pprec(c,c_1), \tripled(x,r,y,c,\ml{main}), \\
  && \naf \testf(\nrel(x,s,y,c)).\\[0.5ex]                     
  \mbox{(ovr-subrc)}  & 
  \ovr(\subRChain,x,y,z,r,s,t,c_1,c) \rif &   \defsubrc(r,s,t,c_1),
                \pprec(c,c_1), \tripled(x,r,y,c,\ml{main}),\\
                && \tripled(y,s,z,c,\ml{main}),  \naf \testf(\nrel(x,t,z,c)).\\[0.5ex]                     
  \mbox{(ovr-dis)}  & 
  \ovr(\pDis,x,y,r,s,c_1,c) \rif &
  \defdis(r,s,c_1), \pprec(c,c_1), \tripled(x,r,y,c,\ml{main}),\\ 
	&& \tripled(x,s,y,c,\ml{main}).\\
  \mbox{(ovr-inv1)}  & 
  \ovr(\pInv,x,y,r,s,c_1,c) \rif &
  \definv(r,s,c_1), \pprec(c,c_1), \tripled(x,r,y,c,\ml{main}), \\
  && \naf \testf(\nrel(x,s,y,c)).\\[0.5ex]
  \mbox{(ovr-inv2)} & \ovr(\pInv,x,y,r,s,c_1,c) \rif &
  \definv(r,s,c_1), \pprec(c,c_1), \tripled(y,s,x,c,\ml{main}), \\
  &&  \naf \testf(\nrel(x,r,y,c)).\\[0.5ex]
  \mbox{(ovr-irr)}  & 
  \ovr(\pIrr,x,R,c_1,c) \rif &
  \defirr(r,c_1), \pprec(c,c_1), \tripled(x,r,x,c,\ml{main}).\\[1.5ex]
\end{array}$}
\hrule
\end{table}

\begin{table}[tp]%

\bigskip

\caption{Deduction rules $P_{\default}$ for defeasible axioms: inheritance rules}
\label{tab:inheritance-rules-tgl}

\medskip

\hrule\mbox{}\\[1ex]
\scalebox{.9}{
\small
$\begin{array}{l@{\;}r@{\ }r@{\ }l@{}}
   \mbox{(prop-inst)} 
   &   \instd(x,z,c,t) & \rif & \insta(x,z,c_1,t), \pprec(c,c_1), \naf \ovr(\insta,x,z,c_1,c).\\[0.5ex]
   \mbox{(prop-triple)} 
   &   \tripled(x,r,y,c,t) & \rif & \triplea(x,r,y,c_1,t), \pprec(c,c_1), \naf \ovr(\triplea,x,r,y,c_1,c).\\[0.5ex]
   \mbox{(prop-ninst)} 
   &   \unsat(t) & \rif & \ninsta(x,z,c_1,t), \instd(x,z,c,t), \\
   &&& \pprec(c,c_1), \naf \ovr(\ninsta,x,z,c_1,c).\\[0.5ex]
   \mbox{(prop-ntriple)} 
   &   \unsat(t) & \rif & \ntriplea(x,r,y,c_1,t),
   \tripled(x,r,y,c,t),\\
   &&&  \pprec(c,c_1), \naf \ovr(\ntriplea,x,r,y,c_1,c).\\[0.5ex]
   \mbox{(prop-subc)} 
   &    \instd(x,z,c,t) & \rif & \subClass(y,z,c_1), \instd(x,y,c,t),\\
    &&& \pprec(c,c_1), \naf \ovr(\subClass,x,y,z,c_1,c).\\[0.5ex]
   \mbox{(prop-cnj}) 
   & \instd(x,z,c,t) & \rif & \subConj(y_1,y_2,z,c_1), \instd(x,y_1,c,t),\instd(x,y_2,c,t),\\
                          &&& \pprec(c,c_1), \naf  \ovr(\subConj,x,y_1,y_2,z,c_1,c). \\[0.5ex]
\mbox{(prop-subex)} 
   & \instd(x,z,c,t) & \rif & \subEx(v,y,z,c_1), \tripled(x,v,x',c,t),\instd(x',y,c,t),\\
                          &&& \pprec(c,c_1), \naf \ovr(\subEx,x,v,y,z,c_1,c). \\[0.5ex]
   \mbox{(prop-supex)} 
   & \tripled(x,r,x',c,t) & \rif & \supEx(y,r,x',c_1), \instd(x,y,c,t),\\
                          &&& \pprec(c,c_1), \naf \ovr(\supEx,x,y,r,x',c_1,c). \\[0.5ex]
\mbox{(prop-forall)} 
   & \instd(y,z',c,t) & \rif & \supForall(z,r,z',c_1), \instd(x,z,c,t), \tripled(x,r,y,c,t),\\
                          &&& \pprec(c,c_1), \naf \ovr(\supForall,x,y,z,r,z',c_1,c).\\[0.5ex]      
\mbox{(prop-leqone)} & \unsat(t) & \rif & \supLeqOne(z,r,g), \instd(x,z,c,t),\\													
                        &&& \tripled(x,r,x_1,c,t), \tripled(x,r,x_2,c,t),\\
                        &&& \pprec(c,c_1), \naf \ovr(\supLeqOne,x,x_1,x_2,z,r,c_1,c).\\[0.5ex]
   \mbox{(prop-subr)} 
   & \tripled(x,w,x',c,t) & \rif & \subRole(v,w,c_1), \tripled(x,v,x',c,t), \\
                          &&& \pprec(c,c_1), \naf \ovr(\subRole,x,y,v,w,c_1,c). \\[0.5ex]
\mbox{(prop-subrc)} 
   & \tripled(x,w,z,c,t) & \rif & \subRChain(u,v,w,c_1), \tripled(x,u,y,c,t), \tripled(y,v,z,c,t),\\
  &&& \pprec(c,c_1), \naf \ovr(\subRChain,x,y,z,u,v,w,c_1,c). \\[0.5ex]
\mbox{(prop-dis)} 
   & \unsat(t) & \rif & \pDis(u,v,c_1), \tripled(x,u,y,c,t), \tripled(x,v,y,c,t),\\
                          &&& \pprec(c,c_1), \naf \ovr(\pDis,x,y,u,v,c_1,c). \\[0.5ex]             
   \mbox{(prop-inv1)} 
   & \tripled(y,v,x,c,t) & \rif & \pInv(u,v,c_1), \tripled(x,u,y,c,t), \\
                          &&& \pprec(c,c_1), \naf \ovr(\pInv,x,y,u,v,c_1,c). \\
   \mbox{(prop-inv2)} 
   & \tripled(x,u,y,c,t) & \rif & \pInv(u,v,c_1), \tripled(y,v,x,c,t), \\
                         &&& \pprec(c,c_1), \naf \ovr(\pInv,x,y,u,v,c_1,c). \\[0.5ex]              
 \mbox{(prop-irr)} 
   & \unsat(t) & \rif & \pIrr(u,c_1),\tripled(x,u,x,c,t),\\
                      &&& \pprec(c,c_1), \naf \ovr(\pIrr,x,u,c_1,c). \\[1.5ex]
\end{array}$}
\hrule
\end{table}

\begin{table}[tp]%
\caption{Deduction rules $P_{\default}$ for defeasible axioms: test rules}
\label{tab:test-rules-tgl} 

\bigskip

\hrule\mbox{}\\[1ex]
\scalebox{.82}{
\small
$\begin{array}{l@{\!\!\!\!}r@{\ }r@{\ }l@{}}

   \mbox{(test-inst)} 
   &    \test(\nlit(x,y,c)) & \rif & \definst(x,y,c_1), \pprec(c,c_1).\\
   \mbox{(constr-inst)} 
   &  & \rif & \testf(\nlit(x,y,c)), \ovr(\insta,x,y,c_1,c).\\[0.5ex]
   \mbox{(test-triple)} 
   &    \test(\nrel(x,r,y,c)) & \rif & \deftriple(x,r,y,c_1), \pprec(c,c_1).\\[0.5ex]
   \mbox{(constr-triple)} 
   &  & \rif & \testf(\nrel(x,r,y,c)), \ovr(\triplea,x,r,y,c_1,c).\\[1ex]
   \mbox{(test-subc)} 
   &    \test(\nlit(x,z,c)) & \rif & \defsubs(y,z,c_1), \instd(x,y,c,\ml{main}), \pprec(c,c_1).\\[0.5ex]
   \mbox{(constr-subc)} 
   &  & \rif & \testf(\nlit(x,z,c)), \ovr(\subClass,x,y,z,c_1,c).\\[1ex]
   \mbox{(test-subcnj)} 
   &    \test(\nlit(x,z,c)) & \rif & \defsubcnj(y_1,y_2,z,c_1), 
	\instd(x,y_1,c,\ml{main}),\\ 
   &&& \instd(x,y_2,c,\ml{main}), \pprec(c,c_1).\\[0.5ex]
   \mbox{(constr-subcnj)} 
   &  & \rif & \testf(\nlit(x,z,c)), \ovr(\subConj,x,y_1,y_2,z,c_1,c).\\[1ex]
   \mbox{(test-subex)} 
   &    \test(\nlit(x,z,c)) & \rif & \defsubex(r,y,z,c_1), 
	 \tripled(x,r,w,c,\ml{main}), \\
   &&&  \instd(w,y,c,\ml{main}),\pprec(c,c_1).\\[0.5ex]
   \mbox{(constr-subex)} 
   &  & \rif & \testf(\nlit(x,z,c)), \ovr(\subEx,x,r,y,z,c_1,c).\\[1ex]
   \mbox{(test-supex)} 
   &    \test(\nrel(x,r,w,c)) & \rif & \defsupex(y,r,w,c_1),\instd(x,y,c,\ml{main}),\pprec(c,c_1).\\[0.5ex]
   \mbox{(constr-supex)} 
   &  & \rif & \testf(\nrel(x,r,w,c)), \ovr(\supEx,x,r,y,w,c_1,c).\\[1ex]
   \mbox{(test-supforall)} 
   &    \test(\nlit(y,w,c)) & \rif & \defsupforall(z,r,w,c_1), \instd(x,z,c,\ml{main}), \\
   &                        &      & \tripled(x,r,y,c,\ml{main}),\pprec(c,c_1).\\[0.5ex]
   \mbox{(constr-supforall)} 
   &  & \rif & \testf(\nlit(y,w,c)), \ovr(\supForall,x,y,z,r,w,c_1,c).\\[1ex]   
   \mbox{(test-subr)} 
   &    \test(\nrel(x,s,y,c)) & \rif & \defsubr(r,s,c_1), \tripled(x,r,y,c,\ml{main}), \pprec(c,c_1).\\[0.5ex]
   \mbox{(constr-subr)} 
   &  & \rif & \testf(\nrel(x,s,y,c)), \ovr(\subRole,x,r,y,s,c_1,c).\\[1ex]
   \mbox{(test-subrc)} 
   &    \test(\nrel(x,t,z,c)) & \rif & \defsubrc(r,s,t,c_1),
   \tripled(x,r,y,c,\main),\\
   &                          &      & \tripled(y,s,z,c,\main), \pprec(c,c_1).\\[0.5ex]
   \mbox{(constr-subrc)} 
   &  & \rif & \testf(\nrel(x,t,z,c)), \ovr(\subRChain,x,y,z,r,s,t,c_1,c).\\[1ex]
   \mbox{(test-inv1)} 
   &    \test(\nrel(x,s,y,c)) & \rif & \definv(r,s,c_1), \tripled(x,r,y,c,\main), \pprec(c,c_1).\\
   \mbox{(test-inv2)} 
   &    \test(\nrel(y,r,x,c)) & \rif & \definv(r,s,c_1), \tripled(x,s,y,c,\main), \pprec(c,c_1).\\[0.5ex]
   \mbox{(constr-inv1)} 
   &  & \rif & \naf \testf(\nrel(x,s,y,c)), \ovr(\pInv,x,y,r,s,c_1,c).\\
   \mbox{(constr-inv2)} 
   &  & \rif & \naf \testf(\nrel(y,r,x,c)), \ovr(\pInv,x,y,r,s,c_1,c).\\[1ex]
	
   \mbox{(test-fails1)} 
   &  \testf(\nlit(x,z,c)) & \rif & \instd(x,z,c,\nlit(x,z,c)), \naf \unsat(\nlit(x,z,c)).\\
   \mbox{(test-fails2)} 
   &  \testf(\nrel(x,r,y,c)) & \rif & \tripled(x,r,y,c,\nrel(x,r,y,c)), \naf \unsat(\nrel(x,r,y,c)).\\[1ex]

   \mbox{(test-add1)} 
   &  \instd(x, z, c, \nlit(x,z,c)) & \rif & \test(\nlit(x,z,c)).\\
   \mbox{(test-add2)} 
   &  \tripled(x, r, y, c, \nrel(x,r,y,c)) & \rif & \test(\nrel(x,r,y,c)).\\[1ex]

   \mbox{(test-copy1)} 
   &  \instd(x_1, y_1, c, t) & \rif & \instd(x_1,y_1,c,\ml{main}), \test(t).\\
   \mbox{(test-copy2)} 
   &  \tripled(x_1, r, y_1, c, t) & \rif & \tripled(x_1,r,y_1,c,\ml{main}), \test(t).\\[1ex]

\end{array}$}
\hrule
\end{table}

\begin{table}[tp]%

\bigskip

\caption{Deduction rules $P_{\default}$ for defeasible axioms: preference rules}
\label{tab:pref-rules}

\medskip

\hrule\mbox{}\\[1ex]
\scalebox{.82}{
\small
$\begin{array}{l@{\;}r@{\ }r@{\ }l@{}}
   \mbox{(pref-inst)} 
   &   \ovr\level\_\insta(x,y,c,n) & \rif & \ovr(\insta, x, y, c_1, c), \level(c_1, n).\\
   \mbox{(wconst-inst)} 
   &                               & \leftsquigarrow & \ovr\level\_\insta(x,y,c,n).\ [1:n] \\[0.5ex]

   \mbox{(pref-triple)} 
   &   \ovr\level\_\triplea(x,r,y,c,n) & \rif & \ovr(\triplea, x, r, y, c_1, c), \level(c_1, n).\\
   \mbox{(wconst-triple)} 
   &                               & \leftsquigarrow & \ovr\level\_\triplea(x,r,y,c,n).\ [1:n] \\[0.5ex]

   \mbox{(pref-ninst)} 
   &   \ovr\level\_\ninsta(x,y,c,n) & \rif & \ovr(\ninsta, x, y, c_1, c), \level(c_1, n).\\
   \mbox{(wconst-ninst)} 
   &                               & \leftsquigarrow & \ovr\level\_\ninsta(x,y,c,n).\ [1:n] \\[0.5ex]

   \mbox{(pref-ntriple)} 
   &   \ovr\level\_\ntriplea(x,r,y,c,n) & \rif & \ovr(\ntriplea, x, r, y, c_1, c), \level(c_1, n).\\
   \mbox{(wconst-ntriple)} 
   &                               & \leftsquigarrow & \ovr\level\_\ntriplea(x,r,y,c,n).\ [1:n] \\[1ex]

   \mbox{(pref-subc)} 
   &   \ovr\level\_\subClass(x,y,z,c,n) & \rif & \ovr(\subClass, x, y, z, c_1, c), \level(c_1, n).\\
   \mbox{(wconst-subc)} 
   &                               & \leftsquigarrow & \ovr\level\_\subClass(x,y,z,c,n).\ [1:n] \\[0.5ex]

   \mbox{(pref-cnj)} 
   &   \ovr\level\_\subConj(x,y_1,y_2,z,c,n) & \rif & \ovr(\subConj,x,y_1,y_2,z,c_1,c), \level(c_1, n).\\
   \mbox{(wconst-cnj)} 
   &                               & \leftsquigarrow & \ovr\level\_\subConj(x,y_1,y_2,z,c,n).\ [1:n] \\[0.5ex]

   \mbox{(pref-subex)} 
   &   \ovr\level\_\subEx(x,v,y,z,c_1,c,n) & \rif & \ovr(\subEx,x,v,y,z,c_1,c), \level(c_1, n).\\
   \mbox{(wconst-subex)} 
   &                               & \leftsquigarrow & \ovr\level\_\subEx(x,v,y,z,c_1,c,n).\ [1:n] \\[0.5ex]

   \mbox{(pref-supex)} 
   &   \ovr\level\_\supEx(x,y,r,x',c,n) & \rif & \ovr(\supEx,x,y,r,x',c_1,c), \level(c_1, n).\\
   \mbox{(wconst-supex)} 
   &                               & \leftsquigarrow & \ovr\level\_\supEx(x,y,r,x',c,n).\ [1:n] \\[0.5ex]   

   \mbox{(pref-forall)} 
   &   \ovr\level\_\supForall(x,y,z,r,z',c,n) & \rif & \ovr(\supForall,x,y,z,r,z',c_1,c), \level(c_1, n).\\
   \mbox{(wconst-forall)} 
   &                               & \leftsquigarrow & \ovr\level\_\supForall(x,y,z,r,z',c,n).\ [1:n] \\[0.5ex]

   \mbox{(pref-leqone)} 
   &   \ovr\level\_\supLeqOne(x,x_1,x_2,z,r,c,n) & \rif & \ovr(\supLeqOne,x,x_1,x_2,z,r,c_1,c), \level(c_1, n).\\
   \mbox{(wconst-leqone)} 
   &                               & \leftsquigarrow & \ovr\level\_\supLeqOne(x,x_1,x_2,z,r,c,n).\ [1:n] \\[1ex]

   \mbox{(pref-subr)} 
   &   \ovr\level\_\subRole(x,y,v,w,c,n) & \rif & \ovr(\subRole,x,y,v,w,c_1,c), \level(c_1, n).\\
   \mbox{(wconst-subr)} 
   &                               & \leftsquigarrow & \ovr\level\_\subRole(x,y,v,w,c,n).\ [1:n] \\[0.5ex]

   \mbox{(pref-subrc)} 
   &   \ovr\level\_\subRChain(x,y,z,u,v,w,c,n) & \rif & \ovr(\subRChain,x,y,z,u,v,w,c_1,c), \level(c_1, n).\\
   \mbox{(wconst-subrc)} 
   &                               & \leftsquigarrow & \ovr\level\_\subRChain(x,y,z,u,v,w,c,n).\ [1:n] \\[1ex]

   \mbox{(pref-dis)} 
   &   \ovr\level\_\pDis(x,y,u,v,c,n) & \rif & \ovr(\pDis,x,y,u,v,c_1,c), \level(c_1, n).\\
   \mbox{(wconst-dis)} 
   &                               & \leftsquigarrow & \ovr\level\_\pDis(x,y,u,v,c,n).\ [1:n] \\[0.5ex]

   \mbox{(pref-inv)} 
   &   \ovr\level\_\pInv(x,y,u,v,c,n) & \rif & \ovr(\pInv,x,y,u,v,c_1,c), \level(c_1, n).\\
   \mbox{(wconst-inv)} 
   &                               & \leftsquigarrow & \ovr\level\_\pInv(x,y,u,v,c,n).\ [1:n] \\[0.5ex]
   
   \mbox{(pref-irr)} 
   &   \ovr\level\_\pIrr(x,u,c,n) & \rif & \ovr(\pIrr,x,u,c_1,c), \level(c_1, n).\\
   \mbox{(wconst-irr)} 
   &                               & \leftsquigarrow & \ovr\level\_\pIrr(x,u,c,n).\ [1:n]   
\end{array}$}
\hrule
\end{table}

\section{Translation correctness: more details}

Given a CAS-interpretation $\IC_{\CAS} = \stru{\IC, \casmap}$, 
(similarly to the CKR case in~\cite{BozzatoES:18}) 
we can build from its components
a corresponding Herbrand interpretation 
$I(\IC_{\CAS})$ of the program $PK(\CKB)$ as the smallest set of literals containing:

\begin{itemize}
\itemsep=0pt
\item 
all facts of $PK(\CKB)$; 
\item 
 $\instd(a, A, \mlc, \smlmain)$, if $\I(\mlc) \models A(a)$; 
\item 
  $\tripled(a,R,b, \mlc, \smlmain)$, if  $\I(\mlc) \models R(a,b)$;
\item
  each $\ovr$-literal from $\OVR(\IC_{\CAS})$;
\item 
  each literal $l$ with environment $t \neq \smlmain$,
  if $\test(t) \in I(\IC_{\CAS})$ and 
	$l$ is in the head of a rule $r \in \grd(PK(\CKB))$
	with $\Body(r) \subseteq I(\IC_{\CAS})$;
\item
  $\test(t)$, if
  $\testf(t)$ appears in the body of an overriding rule $r$
  in $\grd(PK(\CKB))$ and the head of $r$ is an $\ovr$ literal in
  $\OVR(\IC_{\CAS})$;  
\item
  $\unsat(t) \in I(\IC_{\CAS})$, if
  adding the literal corresponding to $t$ to the local interpretation of its
  context $\mlc$ violates some axiom of the local knowledge $\KB_\mlc$;
\item
  $\testf(t)$, if $\unsat(t) \notin I(\IC_{\CAS})$.
\item
  $\ovr\level(p(\ee), n)$, if the corresponding $\ovr$-literal appears in 
	$\OVR(\IC_{\CAS})$ with $\alpha$ in context $\mlc$ and $\level(\mlc, n) \in PK(\CKB)$.
\end{itemize}
Note that $\unsat(\smlmain)$ is not included in $I(\IC_{\CAS})$.

\begin{lemma}
\label{lem:local-correctness}
Let $\CKB$ be a sCKR in $\SROIQrld$ normal form, then:
	\begin{enumerate}[label=(\roman*).]
	\item 
	  for every (named) justified clashing assumption $\casmap$, 
		the interpretation $S = I(\hat{\IC}(\casmap))$ is an answer set of $PK(\CKB)$;
	\item
	   every answer set $S$ of $PK(\CKB)$ is of the form 
		$S = I(\hat{\IC}(\casmap))$ with $\casmap$ a (named) justified clashing assumption for $\CKB$.
	\end{enumerate}
\end{lemma}
\begin{proof}[Sketch]
	Intuitively, as we are interested in computing the correspondence with 
	(not necessarily optimal) answer sets of $PK(\CKB)$ (namely, of the \emph{rules part}
	of the program, not including weak constraints), 
	the newly added weak constraints rules in $P_\default$ do not influence the
	construction of such answer sets and the result can be 
	proved along the lines of Lemma 6 in~\cite{BozzatoES:18}.
	
	Let us consider $S=I(\hat{\IC}(\casmap))$ defined above
	and the reduct  $G_S(PK(\CKB))$ of $PK(\CKB)$ with respect $S$.
  Note that the NAF literals in $PK(\CKB)$ considered in computing such reduct
	involve instances of $\ovr,\testf$ and $\unsat$.
	We can then proceed to prove the lemma by showing that the answer sets of $PK(\CKB)$ 
coincide with the sets $S=I(\hat{\IC}(\casmap))$ where $\casmap$ is a
justified clashing assumption of $\CKB$.

\smallskip
\noindent\textbf{(i).}\ \ \ Assuming that $\casmap$ is a justified clashing assumption,
 we show that $S = I(\hat{\IC}(\casmap))$ is an answer set of $PK(\CKB)$.
 We first that $S \models G_{S}(PK(\CKB))$, that is for every
 rule instance $r \in G_{S}(PK(\CKB))$ it holds that $S \models r$.  We
 can prove this by examining the possible rule forms that
 occur in $G_{S}(PK(\CKB))$.
 Here we show some representative cases (see also~\cite{BozzatoES:18}):
\begin{itemize}
\item
	 \textbf{(prl-instd):}
    then $\insta(a,A,\mlc, t) \in I(\hat{\IC}(\casmap))$ and, by definition of the
		translation, $A(a) \in \KB_\mlc$ (as $t$ can only be $\smlmain$). 
		This implies that $\I(\mlc) \models A(a)$
		and thus\linebreak 
		$\instd(a,A,\mlc,\smlmain)$ is added to $I(\hat{\IC}(\casmap))$.
	\item
	   \textbf{(prl-subc):}
    then $\{\subClass(A,B,\mlc), \instd(a,A,\mlc, t)\} \subseteq I(\hat{\IC}(\casmap))$.
		By definition of the translation we have $A \subs B \in \KB_\mlc$.
		For the construction of $I(\hat{\IC}(\casmap))$, 
		if $t=\smlmain$ then $\I(\mlc) \models A(a)$. 
		This implies that $\I(\mlc) \models B(a)$ 
		and $\instd(a,B,\mlc,t)$ is added to $I(\hat{\IC}(\casmap))$.
		Otherwise, if $t \neq \smlmain$ then
		$\instd(a,B,\mlc,t)$ is directly added to $I(\hat{\IC}(\casmap))$ by its construction.
	\item
	   \textbf{(plc-evalat):}
    then $\{\subEval(A,\mlc_1,B,\mlc), \instd(a,A,\mlc_1, t)\} \subseteq I(\hat{\IC}(\casmap))$.
		Thus we have that $\eval(A,\mlc_1) \subs B \in \KB_\mlc$.
		For the construction of $I(\hat{\IC}(\casmap))$, 
		if $t=\smlmain$ then $\I(\mlc_1) \models A(a)$; 
		This implies that $\I(\mlc) \models B(a)$ 
		and $\instd(a,B,\mlc,t)$ is added to $I(\hat{\IC}(\casmap))$.
		Otherwise, if $t \neq \smlmain$ then
		$\instd(a,B,\mlc,t)$ is directly added to $I(\hat{\IC}(\casmap))$ by its construction.
	\item
	   \textbf{(pref-subc):}
		then $\{\level(\mlc_1, n), \ovr(\subClass,a,A,B,\mlc_1,\mlc)\} \subseteq I(\hat{\IC}(\casmap))$.
		That is, $\ovr(\subClass,a,A,B,\mlc_1,\mlc)$ appears in
		$\OVR(\hat{\IC}(\casmap))$: 
		$\ovr\level\_\subClass(a,A,B,\mlc,n)$ is then added to $I(\hat{\IC}(\casmap))$
		by its construction.
\end{itemize}
Minimality of $S=I(\hat{\IC}(\casmap))$ w.r.t.\ the (positive)
  deduction rules of $G_{S}(PK(\CKB))$ can then be motivated
	as in the original proof in~\cite{BozzatoES:18}:
	thus, $I(\hat{\IC}(\casmap))$ is an answer set of $PK(\CKB)$.

\smallskip
\noindent\textbf{(ii).}\ \ \ 
Let $S$ be an answer set of $PK(\CKB)$.
We show that there is some  justified clashing assumption
$\casmap$ for $\K$ such that $S = I(\hat{\IC}(\casmap))$ holds.

Note that as $S$ is an answer set for the CKR program,
all literals on $\ovr$ and $\testf$ in $S$ are derivable from the
reduct $G_{S}(PK(\CKB))$. 
By the definition of $I(\hat{\IC}(\casmap))$ we can easily build 
a model $\IC_S = \stru{\IC_S, \chi_S}$ from the answer set $S$ as follows:
 for every $c \in \N$,
  we build the local interpretation $\I_S(c)= \stru{\Delta_c, \cdot^{\I(c)}}$ as follows:  
  \begin{itemize}
  \itemsep=0pt
  \item 
    $\Delta_c = \{d \;|\; d \in \NI \}$;  
  \item
	  $a^{\I(c)} = a$, for every $a \in \NI$;
	\item
	  $A^{\I(c)} = \{d \in \Delta_c \mid S \models \instd(d, A, c,\smlmain) \}$, for every $A \in \NC$;
	\item
		$R^{\I(c)} = \{(d,d') \in \Delta_c \times \Delta_c \,|\, 
		S \models \tripled(d, R, d', c,\smlmain) \}$ for $R \in \NR$;
  \end{itemize}
	Finally, $\casmap_S(c) = \{\stru{\alpha, \ee} \mid I_{rl}(\alpha, c') = p, \ovr(p(\ee),c) \in S \}$.
  We have to show that 
  $\IC_{S}$ meets the definition of a least justifed CAS-model for $\CKB$, that is:
  \begin{enumerate}[label=(\roman*)]
\itemsep=0pt
 \item
   for every $\alpha \in \KB_\mlc$ (strict axiom), 
    and $\mlc'\preceq\mlc$, $\I_S(\mlc') \models \alpha$;
  \item
    for every $\default(\alpha) \in \KB_\mlc$ and $\mlc'\prec \mlc$,
    if $\vec{d} \notin \{ \ee \mid \stru{\alpha,\ee} \in \casmap(\mlc')\}$,
    then $\I_S(\mlc') \models \phi_\alpha(\vec{d})$. 
\end{enumerate}

Condition (i) should be proved in the local case where $\mlc'=\mlc$ and
in the ``propagating'' case where $\mlc'\prec\mlc$. The second case can be shown
as a special case of (ii), where overriding to strict axiom is never applicable.
Thus, considering $\mlc'=\mlc$, we verify the condition 
by showing that, for every $\KB_\mlc$, 
we have $\I(c) \models \KB_\scriptmlm$.
This can be shown by cases considering the form of
	all of the axioms $\beta \in \Lcal_\Sigma,\N$ that can occur in $\KB_c$.
For example (the other cases are similar):
	\begin{itemize}
	\item 
	  Let $\beta = A(a) \in \KB_c$, then, by rule (prl-instd),
	  $S \models \instd(a,A,c,\smlmain)$. 
	  This directly implies that $a^{\I(c)} \in A^{\I(c)}$.
	\item 
	  Let $\beta = A \subs B \in \KB_c$, then
	  $S \models \subClass(A,B,c)$. If $d \in A^{\I(c)}$,
	  then by definition $S \models \instd(d,A,c,\smlmain)$:
	  by rule (prl-subc) we obtain that $S \models \instd(d,B,c,\smlmain)$
	  and thus $d \in B^{\I(c)}$.
	\item 
	  Let $\beta = \eval(A, \{\mlc_1\}) \subs B \in \KB_c$, then
	  $S \models \subEval(A,\mlc_1,B,c)$. 
	  If $d \in A^{\I(\mlc_1)}$,
	  then by definition $S \models \instd(d,A,\mlc_1,\smlmain)$
	  and $S \models \instd(c',\ml{C},\ml{gm},\smlmain)$.
	  By rule (plc-evalat) we obtain that $S \models \instd(d,B,c,\smlmain)$:
	  hence, by definition $d \in B^{\I(c)}$.	  
  \end{itemize}		
  To prove condition (ii), let us assume that $\default(\beta) \in \KB_c'$ with $\mlc \prec \mlc'$.
  We can proceed by cases on the possible forms of $\beta$ as in the original proof in~\cite{BozzatoES:18},
	by considering the propagation along the coverage relation. For example:
	  \begin{itemize}
	\item 
	  Let $\beta = A(a)$. Then, by definition of the translation,
	  we have that $S \models \insta(a, A, \mlc',\smlmain)$.
	  Suppose that 
		$\Pair{A(x)}{a} \notin \casmap_{S}(\mlc)$.
	  Then by definition, $\ovr(\insta, a, A, \mlc', \mlc) \notin OVR(\hat{\IC}(\casmap))$.
		Note that we have $S \models \pprec(\mlc, \mlc')$ by construction.
	  By the definition of the reduction, the corresponding instantiation
	  of rule (prop-inst) has not been removed from $G_{S}(PK(\CKB))$: this implies that
	  $S \models \instd(a, A, \mlc,\smlmain)$.
	  By definition, this means that $a^{\I(\mlc)} \in A^{\I(\mlc)}$.
	\item 
	  Let $\beta = A \subs B$. Then, by definition of the translation,
	  we have that $S \models \subClass(A, B, \mlc')$.
		As above, we also have $S \models \pprec(\mlc, \mlc')$.
	  Let us suppose that $b^{\I(\mlc)} \in A^{\I(\mlc)}$: then 
	  $S \models \instd(b, A, \mlc, \smlmain)$.
	  Suppose that $\Pair{A \subs B}{b} \notin \casmap_{S}(\mlc)$:
	  by definition,\linebreak
		$\ovr(\subClass, b, A, B, \mlc', \mlc) \notin OVR(\hat{\IC}(\casmap))$.
	  By the definition of the reduction, the corresponding instantiation
	  of rule (prop-subc) has not been removed from $G_{S}(PK(\CKB))$: 
	  this implies that $S \models \instd(b, B, \mlc, \smlmain)$.
	  Thus, by definition, this means that $b^{\I(\mlc)} \in B^{\I(\mlc)}$.
  \end{itemize}
We have shown that $\IC_S$ is a CAS-model of $\CKB$:
using the same reasoning in the original proof in~\cite{BozzatoES:18}
we can also prove the $\IC_S$ corresponds to the least model 
and that $\chi_S$ is justified, thus proving the result.
\end{proof}

\begin{lemma}
\label{lem:wc-correctness}
  Let $\CKB$ be a sCKR in $\SROIQrld$ normal form with ranked context hierarchy.
	Then, $\hat{\IC}$ is a CKR model of $\CKB$ iff
	there exists a (named) justified clashing assumption $\casmap$ s.t.
	$I(\hat{\IC}(\chi))$ is an optimal answer set of $PK(\CKB)$.
\end{lemma}
\begin{proof}[Sketch]
	To prove the result, we have to show that, $\hat{\IC}$ is a CKR model iff:
	\begin{enumerate}[label=(\roman*)]
	\item 
	  there exists a (named) justified clashing assumption $\casmap$ s.t.
	  $I(\hat{\IC}(\chi))$ is an answer set of $PK(\CKB)$.
	\item
	  $I(\hat{\IC}(\chi))$ is an optimal answer set of $PK(\CKB)$.
	\end{enumerate}
	Condition (i) is directly derived from Lemma~\ref{lem:local-correctness}
	and the definition of CKR model in Definition 10.
	
	To prove (ii), we have to show the correspondence of the
	lexicographic order on global profiles $p(\chi)$ with the order induced by
	objective function $H^{PK(\CKB)}(S)$ on answer sets. That is,
	$I(\hat{\IC}(\chi))$ is optimal iff
	there does not exist a justified $\chi'$ s.t. $p(\chi') < p(\chi)$.
	
	First of all, we note that weak constraints are only associated to 
	instances of overridings (i.e. $\ovr$ atoms): thus the optimization
	of the answer sets is only dependent on minimization of 
	aspects related to such atoms (which, on the other hand, are related
	to the clashing assumptions in $\chi$).
	
	Suppose that $\chi$ is preferred, that is 
	there does not exist a justified $\chi'$ s.t. $p(\chi') < p(\chi)$.
	Thus, for every such $\chi'$ we have $p(\chi') > p(\chi)$.
	By the definition of lexicographic order on profiles, this means that
	if $p(\chi) = (l_n, \dots, l_0)$ and $p(\chi') = (l'_n, \dots, l'_0)$
	some $j \in \{0,\ldots,n\}$ exists such that 
  $l_n = l'_n$, $l_{n-1} = l'_{n-1}$, \ldots 
  $l_{j+1} = l'_{j+1}$, and $l_j < l'_j$.
	This means that there exist at least an ``additional''
	$\stru{\alpha,\ee} \in \casmap'(\mlc)$ for a context $\mlc$
 such that $l(\alpha) = j$. 
  That is, either all elements in $\chi$ have level smaller than $j$
	or $\chi'$ has more elements at the level $j$.
	Considering then the interpretation $S'= I(\hat{\IC}(\chi'))$, 
	can show that it necessarily has an higher cost with respect 
	$S= I(\hat{\IC}(\chi))$.
	Since $\stru{\alpha,\ee} \in \casmap'(\mlc)$, by construction of $S'$ we have
	that the corresponding $\ovr(p(\ee)) \in S'$ and 
	$\ovr\level(p(\ee),j) \in S'$:
	this causes the instantiation of the weak constraint
	rule in relative to $\ovr\level(p(\ee),j)$, which
	adds a weak constraint violation to $S'$ at level $j$
	and with cost $1$.
	Considering the definition of the optimization function
	$H^{PK(\CKB)}$ from~\cite{LeonePFEGPS:02}:
	\begin{itemize}
	\item 
	  if the violation in $S'$ is at a level bigger than all
		of the violations in $S$, the level function $f_{PK(\CKB)}(j)$
		in the definition of $H^{PK(\CKB)}$ 
		is assured to add an higher cost than all of the lower levels
		$f_{PK(\CKB)}(i)$;
	\item
	  if the violation in $S'$ is at the same level of the (higher) 
		violation in $S$, then the additional cost $1$ of the violation assures
		that level cost of $j$ in $S'$ is bigger than in $S$.
	\end{itemize}
	Thus, we have that in both case $H^{PK(\CKB)}(S') > H^{PK(\CKB)}(S)$.
	This shows the optimality of $I(\hat{\IC}(\chi))$.
	
	The other direction can be shown similarly: supposing that
	$S= I(\hat{\IC}(\chi))$ is optimal, then 
	for all other $S'= I(\hat{\IC}(\chi'))$ we have 
	$H^{PK(\CKB)}(S') > H^{PK(\CKB)}(S)$.
	Thus, by the definition of the function, we have that
	there exists at least a violation on a $\ovr(p(\ee))$
	with higher level or higher level cost at a level $j$.
	Considering the corresponding clashing assumption sets,
	we can analogously map back to the definition of lexicographic ordering 
	on profiles, obtaining that $p(\chi') > p(\chi)$.
	Thus, $\chi$ is preferred and we proved the result.
\end{proof}



  

\newpage
\bibliographystyle{splncs03}
\bibliography{bibliography}


\end{document}